\documentclass{bmvc2k}
\usepackage{float}
\usepackage[space]{grffile}
\usepackage{enumitem}
\definecolor{ocre}{RGB}{0,0,100}
\usepackage[font={color=ocre},justification=centering]{caption}
\usepackage{amsmath}
\usepackage{amssymb}
\usepackage{amsthm}
\newcommand{\cmmnt}[1]{}
\usepackage{graphicx}
\newtheorem{theorem}{Theorem}
\newtheorem{corollary}{Corollary}[theorem]

\setlist{nolistsep}

\title{Investigating Convolutional Neural Networks using Spatial Orderness}

\addauthor{Rohan Ghosh}{rghosh92@gmail.com}{1}
\addauthor{Anupam K.Gupta}{anupamgupta1984@gmail.com}{1}
\addauthor{Mehul Motani\textsuperscript{1,}}{motani@nus.edu.sg}{2}

\addinstitution{The N.1 Institute of Health, National University of Singapore}
\addinstitution{Department of Electrical and Computer Engineering, National University of Singapore}

\runninghead{Spatial Orderness}{BMVC}

\begin{document}

\maketitle

\begin{abstract}
%  It is common knowledge that more often than not, deeper convolutional neural networks perform better. The important reason cited for improvements from increase of network depth involve the ability to compute more complex network representations. But it has been shown that shallow networks can compute functions just as complex as deep networks (cite). Then what drives deep networks to be so efficient ? 

Convolutional Neural Networks (CNN) have been pivotal to the success of many state-of-the-art classification problems, in a wide variety of domains (for e.g. vision, speech, graphs and medical imaging). A commonality within those domains is the presence of hierarchical, spatially agglomerative local-to-global interactions within the data. For two-dimensional images, such interactions may induce an \textit{a priori} relationship between the pixel data and the underlying spatial ordering of the pixels. For instance in natural images, neighboring pixels are more likely contain similar values than non-neighboring pixels which are further apart. To that end, we propose a statistical metric called \textit{spatial orderness}, which quantifies the extent to which the input data (2D) obeys the underlying spatial ordering at various scales. In our experiments, we mainly find that adding convolutional layers to a CNN could be counterproductive for data bereft of spatial order at higher scales. We also observe, quite counter-intuitively, that the spatial orderness of CNN feature maps show a synchronized increase during the intial stages of training, and validation performance only improves after spatial orderness of feature maps start decreasing. Lastly, we present a theoretical analysis (and empirical validation) of the spatial orderness of network weights, where we find that using smaller kernel sizes leads to kernels of greater spatial orderness and vice-versa.

\end{abstract}

%%%%%%%%% BODY TEXT
\section{Introduction}

 There has been a large body of theoretical and experimental work exploring various attributes of CNNs which may contribute towards their excellent performance and generalization abilities (\cite{zhou_generalization_icml,depth_aaai_sun,kawaguchi_generalization,ben_jio_representation,ben_jio_scaling_ai,eff_recep_cnn_nipsy,fergus_vis_cnn}. However, the unusual effectivness of CNNs on a large variety of domains (vision, audio, graphs, medical imaging) is still not entirely comprehended.  
 
 Solely from the perspective of a mathematical function, it is intruiging to see a convolutional neural network demonstrate significant performance gains, when compared to a fully connected deep neural network. It is also clear that CNNs and fully-connected neural networks (FC-NNs) showcase very different behaviour. For instance, enough empirical evidence exists \cite{residual_he,vgg_first,ben_jio_scaling_ai} to suggest that increasing depth in CNNs (by adding convolution layers) almost always leads to considerable performance improvements, on most problems. However, increasing the depth of a FC-NN quickly leads to worse test performance \cite{depth_aaai_sun}. Unlike a FC-NN, CNN exhibits translation equivariance across its layers, that enable it to achieve translation equivariant representations deep within the network. This is useful for applications where global translational symmetries exist in the data. However, even in problems where global translational symmetries do not exist (MNIST for instance), CNNs easily outperform FC-NNs \cite{dropconn_lecun}. 
 
Compared to structure-less FC-NNs, the inductive biases in a CNN are clearly better suited to handle classification problem in various domains. But, instead of a function-based introspection of a CNN, we ask which characteristics of the data itself enable these inductive biases to flourish? Would adding convolution layers still be profitable if one were to synthetically distort these "convolution-conducive" characteristics in the data ?

In this work, we systematically explore these questions, based on a hypothesis: \textit{Convolutional structure in a neural network benefits from spatially ordered data.} We define \textit{spatial order} to be the extent to which spatial proximity determines data value proximity. For images, an example of high spatial order in data is when spatially nearby pixels are more likely to have similar values than pixels which are far apart. Similarly, when the pixel intensity differences are independent of the spatial distance between the pixels (e.g. in white-noise images), the data is said to have low spatial order. Spatially ordered data is likely to contain more meaningful spatial structure and hierarchy, and can benefit from locality-preserving feedforward functions; like convolutional operations in CNNs.  
\footnote{For graphs, we can extend the definition of spatial order to one of locality. For instance, a graph where every node is connected to every other node would be a counter-example of locality in a graph.} 

A simple, novel metric for reliably quantifying spatial order in 2D data is proposed, denoted as \textit{spatial orderness}. This metric can either be computed for a single 2D image, or for an entire dataset of 2D images. First, it is shown that spatial orderness is a reliable quantifier of spatial order in the input: synthetic disruption of spatial structure in the data decreases spatial orderness. Next, we find that adding convolutional depth to a CNN ceases to yield performance improvements, when the data lacks spatial order at higher scales (Figure \ref{fig:cnn_swap_depth}). The remaining experiments and theoretical contributions of this work relates to the computation of spatial orderness of the (a) input, (b) feature maps (during and after training) and (c) kernels (post-training).

\section{Spatial Diffusability implies Spatial Orderness}

Let us denote a set of 2D data as $I_1,I_2,..,I_N$, all of equal size. We denote underlying data generating probability distribution as $D$. Consider a spatial location $p$ within the domain of $I$, such that $I_j(p)$ denotes the value of the image $I_j$ at the spatial location $p$. Let us denote an image $X$ which is a random variable following the distribution $D$. Since $X$ is a random variable, it follows that $X(p)$, which represents the value of $X$ at spatial location $p$, is also a random variable. Proceeding with these definitions, we outline our approach for computing spatial orderness. 

For the purpose of quanitfying spatial order, we propose a spatial diffusion based generative modelling of $X$. The key observation is that nearby spatial locations (or regions) must show smaller differences in intensity values (or average intensity), \textit{compared} to spatial locations (or regions) which are further apart. Note that simple correlations between neighboring pixel values is not enough to concretely quantify the above relationship.  

We begin with three random variables extracted from different spatial locations within $X$: $X(p)$, $X(q)$ and $X(r)$. Importantly, the spatial locations $p$, $q$ and $r$ are chosen such that $q \in N(p)$ and $r \in N(N(p)) \setminus N(p)$, where $N(p)$ represents the set of neighboring spatial locations of $p$. Here $\setminus$ is the set subtraction operator, such that $r$ is at a distance of two hops from $p$. Throughout this paper, we denote this particular spatial arrangement of locations as the \textit{two-hop spatial arrangement}. With this, we can outline the relationship between the random variables $X(p)$, $X(q)$ and $X(r)$ using a normally distributed spatial diffusion process:
\begin{align}
X(q) &= X(p) + \mathcal{N}(0,\sigma) \label{eq:first_hop} \\  
X(r) &= X(q) + \mathcal{N}(0,\sigma). \label{eq:second_hop} 
\end{align}
Combining equations \ref{eq:first_hop} and \ref{eq:second_hop}, we have
\begin{equation} \label{eq:big_hop}
    X(r) = X(p) + \mathcal{N}(0,\sqrt{2}\sigma), 
\end{equation}
using which the above equations can be summarized with the relationship:
\begin{equation} \label{eq:orderness}
    \mathbb{E}\left[(X(p)-X(r))^2\right] = 2\times \mathbb{E}\left[(X(p)-X(q))^2\right],
\end{equation}
which is independent of $\sigma$. Note that equation \ref{eq:orderness} shows an asymmetric relationship between $X(p)$, $X(q)$ and $X(r)$, due to their spatial positioning. On the other hand, if the  $X(p)$, $X(q)$ and $X(r)$ are randomly permuted, the spatial diffusion model doesn't hold. This is because by doing so, we essentially remove any relationship between the pixel values and the pixel locations. In that case, the effect of random permutations results in 
\begin{equation} \label{eq:disorderness}
    \mathbb{E}\left[(X(p)-X(r))^2\right] =  \mathbb{E}\left[(X(p)-X(q))^2\right] = \mathbb{E}\left[(X(q)-X(r))^2\right],
\end{equation}
which is a symmetric relationship between $X(p)$, $X(q)$ and $X(r)$. Thus equation \ref{eq:orderness} and \ref{eq:disorderness} represent opposite extremes of high spatial orderness and low spatial orderness respectively. We now have all the necessary tools to define the spatial orderness metric, and extend it for multiple scales. It is done in the following section. 
\section{Multi-Scale Spatial Orderness} 

Given a set of images, $I_1,I_2,...,I_k$, we first extract a fixed number ($l$) of triples of pixel intensities $(I_{n(1)}(p_1),I_{n(1)}(q_1),I_{n(1)}(r_1)),...,(I_{n(l)}(p_l),I_{n(l)}(q_l),I_{n(l)}(r_l))$, such that each triple of spatial locations $(p_i,q_i,r_i)$ follows a 2-hop spatial arrangement. Here $n(i)$ denotes the image from which the $i^{th}$ triple was extracted. Next, we define the spatial orderness measure at the lowest scale as follows, 
\begin{equation}
so(I)^{1} = \left( \frac{\mathbb{E}_i \left[ \left( I_{n(i)}(p_i)-I_{n(i)}(r_i) \right)^2 \right]}{ \mathbb{E}_i \left[ \left( I_{n(i)}(p_i)-I_{n(i)}(q_i) \right)^2 \right]} \right) - 1.
\end{equation}

 Observe that $so(I)^{1}=1$, when the diffusion process $P \xrightarrow{} Q \xrightarrow{} R$ is strictly followed (equation \ref{eq:orderness}), whereas it drops down to zero when all spatial order is removed e.g. by random permutation of pixel values. 

With this, we can extend the definition of spatial orderness to multiple spatial scales. For that, a scale-space like decomposition is constructed by averaging $a\times a$ non-overlapping input regions onto a single pixel value. We let these sets of new mean downsampled images be denoted as $I_1^{a},I_2^{a},....,I_k^{a}$. For each set of images at each scale, we denote their corresponding spatial orderness values by $so(I)^{a}$. At the end, we have a set of scalar values $so(I)^{1},so(I)^{2},...,so(I)^{p}$, which represent the spatial orderness of the data at various scales.
We summarize some the ways in which this measure can be interpreted:
\begin{itemize}
    \item Spatial orderness at the lowest scale is indicative of how much more accurately the value of a pixel can be interpolated from its neighbors \textit{than} its non-neighbors which are at a distance of 2 hops.  
    
    \item Randomly permuting the spatial locations of pixels (or blocks of pixels) will reduce spatial orderness. Conversely, when a randomly permuted version of an input has an equal likelihood of occurence to its non-permuted form, the spatial orderness of data is zero at all scales.
    \item 2D inputs of the form $I[i,j] = a(I) + \mathcal{N}(0,\sigma)$ ($a$ is a constant that can be different for different I), have zero spatial orderness. One can observe that random permutation of pixels does not change the image statistics in this case. Also, note that by controlling variation in $a(I)$ one can arbitrarily increase spatial "correlation" measures, hinting that correlation is not enough to capture spatial order in all cases. 
    % \item Consider the input corresponding to a class of textures, which are approximately statistically periodic at a certain spatial scale $s$. It can be shown that the spatial orderness of the textures computed for scales $\geq s$ is small.   
\end{itemize}

\section{Experiments}

The experiments reported herewith are conducted on the MNIST \cite{mnist_orig}, Fashion-MNIST \cite{fashion_mnist} and CIFAR-10 \cite{cifar_orig} datasets. For MNIST and Fashion-MNIST, we perform $2\times2$ max-pooling after each convolution layer, whereas for CIFAR-10, pooling was only performed after each alternate convolution layer. Additionally, a two layer fully connected network is used for CIFAR-10, whereas a single fc layer is used for MNIST and Fashion-MNIST. For consistency we used 64 units in all hidden layers, with kernels of size $3\times 3$, except in section \ref{sec:kerorderness_expt} (variation of kernel size).

\subsection{Disrupting Spatial Orderness: Random Block-Swapping}
Here we describe a method for disrupting the spatial orderness of the data, by performing block-swapping on the input. First we divide $(N\times N)$ images into blocks of size $k\times k$, such that $N/k \times N/k$ blocks span the entire image. Next, in each iteration of block swapping, a random chosen pair of image blocks are entirely swapped. We then repeat this process for  $Ns$ number of iterations. More swaps (larger $Ns$) will lead to a greater disruption of spatial order, and thus should elicit lower values of spatial orderness, and vice-versa. Furthermore, the block size ($k$) of the swap is relevant: swapping for a certain $k$ must not greatly impact the spatial orderness at scales less than $k$, as the spatial arrangement in those scales is not overly affected. \footnote{Block-swapping with larger block-size cannot altogether avoid disrupting the spatial order at lower scales, due to boundary effects of the blocks.}

\subsubsection{Random Block-Swapping: Impact on Spatial Orderness}
\label{sec:block_swap}

\begin{figure}
  \centering
    \includegraphics[width=0.8\textwidth]{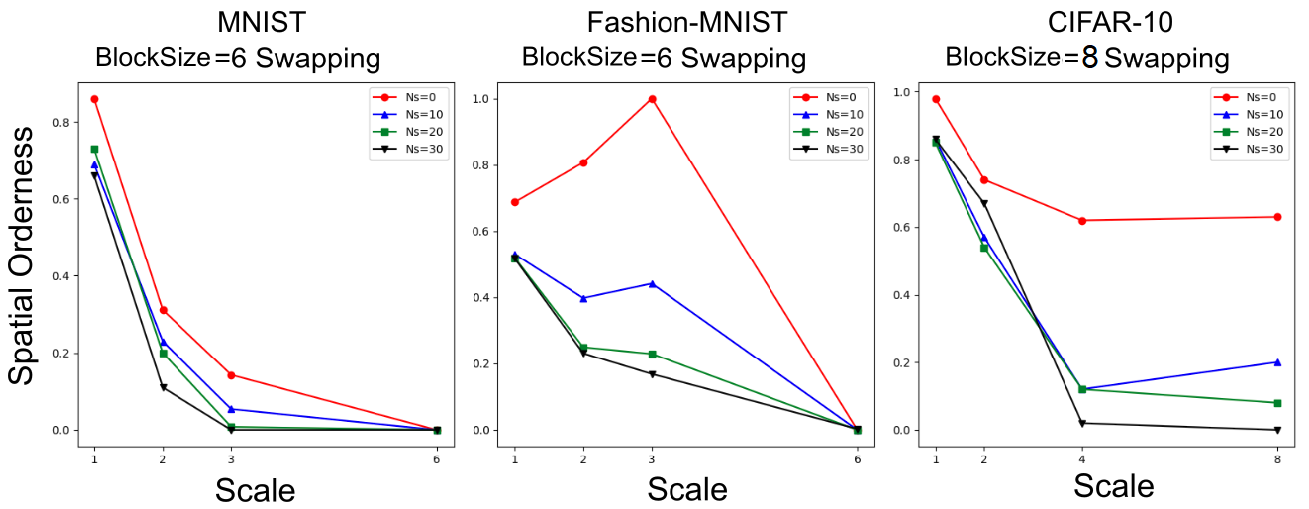}
    \caption{ Spatial orderness of the MNIST, Fashion-MNIST and CIFAR-10 datasets at various scales, and their changes with block-swap operations performed on the data. For instance, the plots in red showcase the spatial orderness of the original, unswapped datasets ($Ns=0$) at specific scales. For each dataset, block swapping was performed for a pre-selected block-size (specified on top). }  
    \label{fig:orderness_swap_all}
\end{figure}

To analyze the effect of block-swapping on spatial orderness measures at various scales, we simply vary the number of block-swap operations on each image of the corresponding datasets. Increasing the number of swaps leads to a steady reduction of spatial order as a whole, in the data. Therefore, a metric which measures spatial order must give smaller values when many block-swap operations are performed on the input. For our experiments, we choose four different number of block-swaps ($Ns=(0,10,20,30)$) for the datasets of MNIST (BlockSize=6), Fashion-MNIST (BlockSize=6) and CIFAR-10 (BlockSize=8), generating a total of 12 datasets: MNIST-swap$_{6}$(0,10,20,30), CIFAR10-swap$_{8}$(0,10,20,30) and  Fashion-MNIST-swap$_{6}$(0,10,20,30).  
The results are shown in figure \ref{fig:orderness_swap_all}.

First, we look at how spatial orderness changes with scale. As expected, we find that in all three datasets, spatial orderness at the highest scale is significantly lower than in the initial scales. This fact re-affirms the apparent "bag-of-words" like organisation of images at higher scales (objects or patterns are more positionally decorrelated at higher scales) (see \cite{bag_of_features_cnn}).

Next, we note the impact that block swapping has on spatial orderness of the corresponding scales. In all cases, we observe a clear reduction (to zero) of spatial orderness with a greater amount of block swaps, at the corresponding scales. Since block swapping is done with relatively larger block-size, the spatial orderness of the data for scales less than the block-size is not overly affected.

\subsubsection{Classification experiments: Is greater convolutional depth always better ?}\label{sec:swap_cnn_depth}

\begin{figure}
  \centering
    \includegraphics[width=0.9\textwidth]{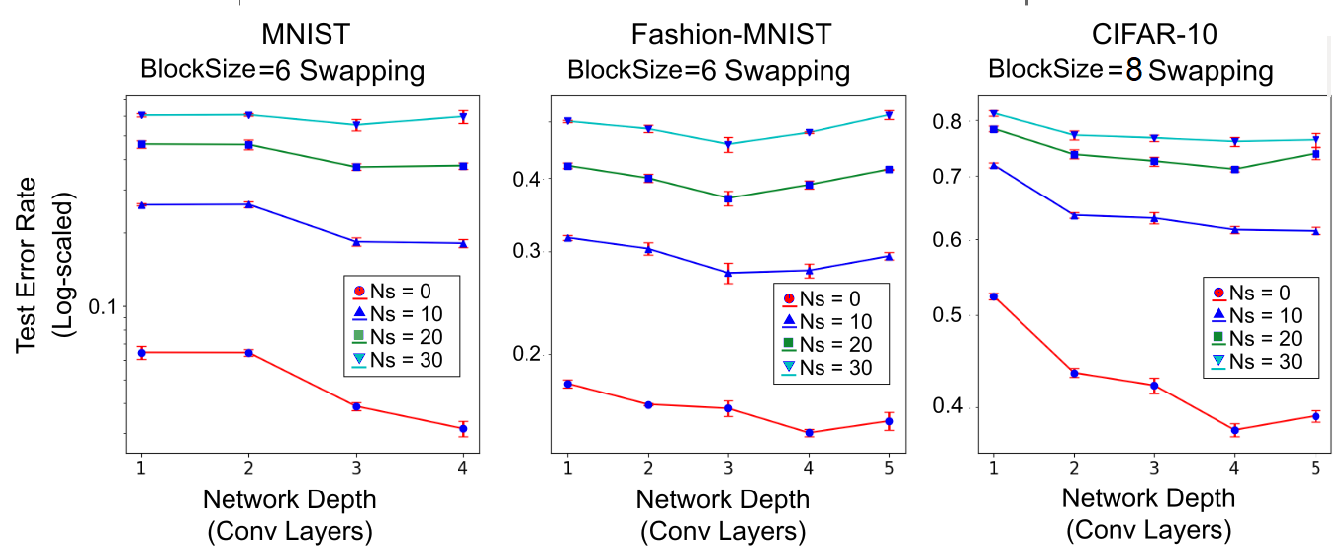}
    \caption{Semilog plots showing the test error rate of networks of different depths, trained on data corrupted by various degrees of spatial block-swapping ($Ns=(0,10,20,30)$) on three different datasets (MNIST, Fashion-MNIST and CIFAR-10). Note that for data lacking in spatial order ($Ns>0$), depth additions beyond a certain point do not yield improvements. Instead, such additions often significantly increase error rate, for larger $Ns$. }  
    \label{fig:cnn_swap_depth}
\end{figure}

Here we document CNN classification performance on MNIST-swap$_{6}$(0,10,20,30), CIFAR10-swap$_{8}$(0,10,20,30) and the Fashion-MNIST-swap$_{6}$(0,10,20,30) datasets. The objective of this experiment is to discover if adding convolutional layers to a CNN is still beneficial, when the spatial orderness of the data has been reduced at higher scales. Our primary hypothesis is that convolution layers exploit the spatial orderness of data at multiple scales. Hence, for block-swapped data, we must expect the addition of convolution layers (beyond the scale of the swap) to pay decreasing dividends. Furthermore, because the block-swaps are only done at a higher scale, we should still find that adding initial convolution layers are beneficial, as spatial orderness of initial scales are still preserved (Figure \ref{fig:orderness_swap_all}). 

 Results are shown in figure \ref{fig:cnn_swap_depth}. As hypothesized, we find that indeed adding convolution layers lead to decreasing gains, for larger number of block-swaps at the corresponding scales (larger $Ns$). Also, as anticipated, we observe that initial additions of convolution layers reduce test errors irrespective of block swapping.
 
 Our findings are consistent with \cite{deep_v_shallow}, where it was theoretically shown that stacked convolution layers are optimally priored for learning compositional functions. Greater convolutional depth implies greater compositionality which effectively spans a range of scales in the input. By removing spatial order at higher scales by random block-swapping, we are essentially disrupting the compositional structure of the data at higher scales, which leads to diminishing improvements with adding more depth.

\subsection{Spatial Orderness of CNN Feature Maps}

 The previous sections demonstrate that convolutions are more effective when the input data has spatial order at multiple scales. However, note that a convolutional module at a depth of $n+1$ does not compute on the input data, but rather on the feature map of the $n^{th}$ layer. Since each convolution layer sees the feature map of the previous layers as its input, the computation of the spatial orderness of the feature maps would be a meaningful step. For each feature response map (denoted as a function  $f(.)$), computation of spatial orderness proceeds in the same way as for 2D images, treating the feature responses across all training examples as the set of 2D images $f(I_1),f(I_2),..f(I_k)$. As convolution outputs usually have multiple feature maps, the spatial orderness of a layer is estimated as the mean value of the spatial orderness of each individual feature map within that layer.

%  To summarize, 
% \begin{itemize}
% \setlist[1]{itemsep=1pt}
%     \item On average, the imageness of the final layer seems to be indicative of the expected improvement in accuracy from convolution layer additions. The table demonstrates the high correlation between the imageness of the final feature map within the CNN, and the mean reduction of error rate (in \%) obtained by the addition of a convolution layer. 
%     \item Increasing the number of block-swaps, has a negative impact on the imageness measure of the final layer feature map. \cmmnt{Thus, decreasing the imageness of the input (results in previous section) gets mirrored onto the feature representations computed within the CNN.} 
% \end{itemize} 

\subsubsection{MNIST and CIFAR-10: Training-time Progression}
\label{sec:training_so_progression}

 \begin{figure}
  \centering
    \includegraphics[width=0.9\textwidth]{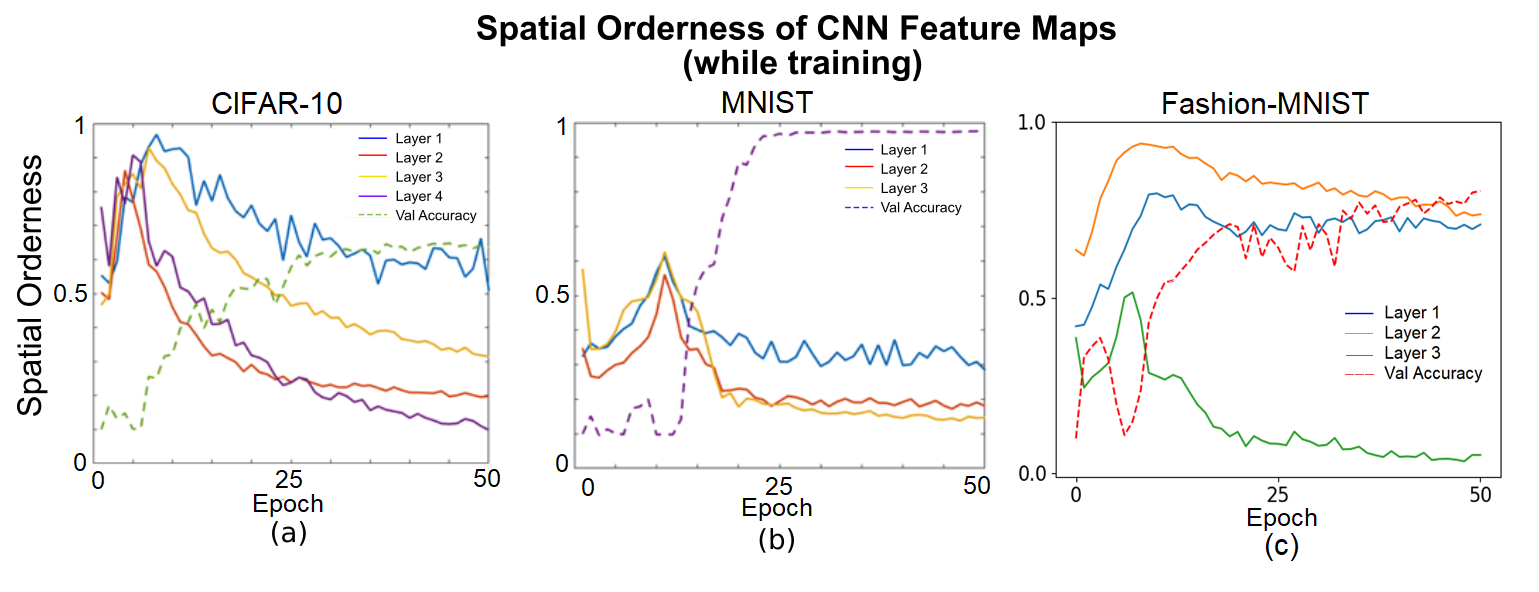}
    \caption{(a), (b) and (c) show the progression of mean spatial orderness of the feature maps (of each layer), while training on CIFAR-10, MNIST and Fashion-MNIST respectively. The validation performance at the end of each epoch is also shown (dotted lines).}  
    \label{fig:imness_seq_corr}
\end{figure}

CNNs are trained on the original datasets of MNIST (3 conv layers), Fashion-MNIST (3 conv layers) and CIFAR-10 (4 conv layers). During training, the validation accuracies and spatial orderness of feature maps at the end of each epoch was monitored. As our initial experiments in Section 
\ref{sec:block_swap} demonstrated, spatial orderness of images at highest scale is usually low, which indicates a bag-of-words like, spatially decorrelated organisation of information at higher scales. Hence the feature maps of CNN can be treated in the same way w.r.t their spatial orderness measures, i.e. low spatial orderness must signify higher levels of abstraction and spatially less redundant information. Concurrently, higher spatial orderness must signify a greater level of spatial redundancy in information, as it essentially implies that spatial gaps in feature values can be effectively "filled in" by means of interpolation using the feature values of the pixel neighbors. 

The spatial orderness progression of feature maps of all the CNN layers are shown in figure \ref{fig:imness_seq_corr} (additional plots available in supplementary material). For both datasets, we find that the spatial orderness measures show a synchronized "peaking" at the beginning of training. This is an unusual finding, as we should expect the features to get higher in abstraction with more training. Curiously, we do observe that the validation accuracy of the networks \textit{do not} show any steady improvement in this phase. Subsequently, after the peak, we find that the spatial orderness of all layers show a synchronized decrease. The spatial orderness of the last layer shows the greatest reduction after its peak; signifying that spatially de-correlated, abstract concepts are learned mainly in the higher layers. This finding is consistent with our current understanding of CNNs: generalization performance only starts improving when the higher layer representations start capturing increasingly abstract features (see particularly (c) in figure \ref{fig:imness_seq_corr}). Interestingly, we also notice that the spatial orderness of the other layers decrease with training as well, even including the first layer (blue trajectories). This indicates that while a CNN does eventually captures spatially non-redundant abstract concepts in its deeper layers, all other layers also sequentially strive to capture spatially non-redundant features of low spatial orderness.

\section{Spatial Orderness of Kernels} \label{sec:kernel_orderness}

\subsection{Theoretical Results} \label{sec:kerorderness_theory} 

We note that just like the inputs and the feature maps, one can treat the kernels (of size $K\times K$) as 2D images themselves. As such, it is also possible to compute the spatial orderness within the kernels, at the end of training. Convolution is linear in nature, and will elicit larger output responses when the input patches are highly correlated to the kernel form. Thus, kernels with very low spatial order are not likely to extract visually meaningful features, and vice-versa. Hence, from a feature extraction point of view, it is desirable that weights exhibit high spatial orderness.

Here we summarize our theoretical results on the spatial orderness of kernels. Please find our main theoretical results (Theorems 1, Corollaries 1.1 and 1.2) and proofs in the supplementary material. We summarize the theorems as follows. 
\begin{itemize}
    \item \textbf{Theorem 1 and Corollary 1.2}: \textbf{How is the spatial orderness of kernels and the spatial orderness of the feature map input related ?} We find that the spatial orderness of the kernels are likely to be higher when the inputs themselves have higher spatial orderness.
    \item \textbf{Corollary 1.1}: \textbf{How is the spatial orderness of kernels related to the choice of kernel size ?} We find that choosing a larger kernel size can lead to kernels with lower spatial orderness\footnote{Note that by "spatial orderness of kernels" we mean the average spatial orderness of post-trained kernel weights (averaged across all kernels within a layer). In the following section, we empirically substantiate the results in the theorems.}. This shows that the choice of kernel size is quite important w.r.t ensuring spatially ordered kernels.
\end{itemize}

\begin{figure}
  \centering
    \includegraphics[width=0.8\textwidth]{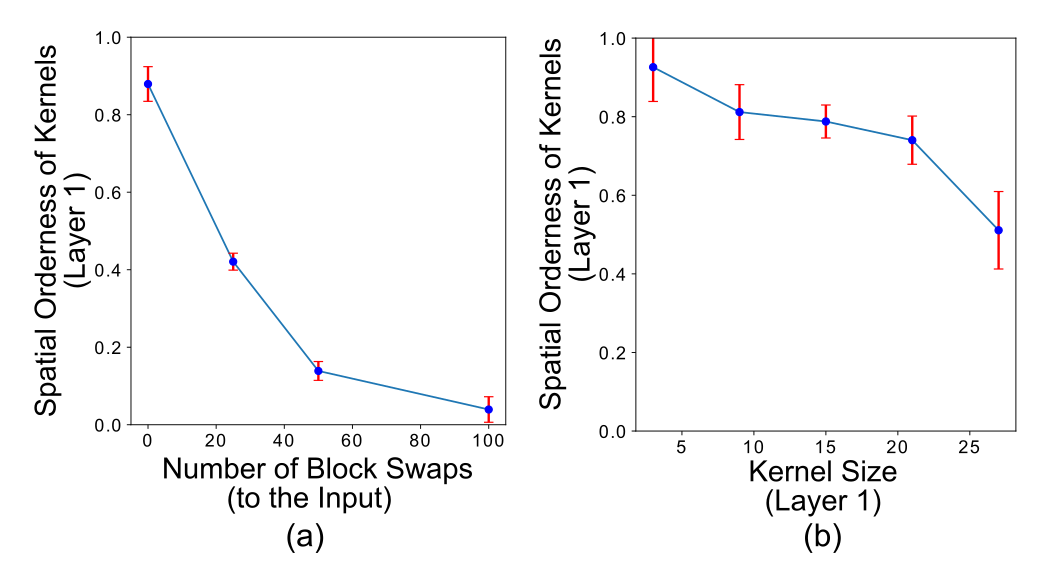}
    \caption{(a) demonstrates that disruption of spatial orderness at the input has an immediate effect on the spatial orderness of the kernels, and (b) shows that the size of kernels affect the spatial orderness of the trained kernels. All experiments were done on MNIST-1000 (1000 training examples used) and each experiment was repeated across six random splits of the data. 
    }
    \label{fig:bs_ks}
\end{figure}

\subsection{Experimental Validation of Results} \label{sec:kerorderness_expt}

A 3-layered CNN was trained on six random splits of the MNIST dataset, with random network intializations each time. Only 1000 examples were used for training. Two separate experiments were conducted, guided by the theoretical results: (a) MNIST-swap$_{3}$(0,25,50,100) datasets were generated from each random split on MNIST, on which the CNNs were trained, and (b) kernel size of the first layer was varied in the range (3,27) (for zero-padded convolutions), and the networks were trained on conventional MNIST. The objective of the first experiment was to examine the relationship between the spatial orderness of the input and the kernels; whereas the second experiment's goal was to discover how the spatial orderness of the kernels was dependent on the choice of kernel size. Note that for the second experiment, the convolution padding of the first layer was zero; i.e. for large kernel sizes the convolution effectively becomes a fully connected layer.

The results are shown in the plots of Figure \ref{fig:bs_ks}. We find that they conform to the predictions of our theoretical results. We observe a sharp reduction in the spatial orderness of kernels when the input data has less spatial order (using greater block-swaps). Furthermore, we find that the kernel spatial orderness reduces with larger kernel size, almost by a factor of half (from K=3 to K=27). Note that in the second experiment, no block-swapping was performed on the input. 

These results add an interesting perspective on the debate of CNNs versus FC-NNs.
Taken together, the results imply that a CNN is more likely to extract visually meaningful (spatially ordered) features, subject to two necessary conditions: (a) the kernel size of the convolutions are small (i.e. more CNN than FC-NN like) and (b) the data on which the network is trained exhibits high spatial orderness.

\section{Discussions: Connection to Other Works}

Recently it was found that on Imagenet, a bag-of-features based approach with shallow CNNs performs surprisingly close to bigger models which exploit spatial structure at higher scales \cite{bag_of_features_cnn}. Hence, spatial arrangement information beyond a certain scale is not very yielding in terms of improving classification performance. This is consistent with our findings in this paper. As demonstrated in section \ref{sec:block_swap}, the spatial orderness of the image at higher scales is usually lower than the spatial orderness at lower scales, i.e. the information at the higher scales is more "bag-of-words" like than in lower scales. Furthermore, experiments in section \ref{sec:training_so_progression} and section 3 of the supplementary material reveal that deeper convolutional feature maps indeed mirror the spatial organization of the input at higher scales.

Another example of testing the generalization abilities of CNNs on the data is in \cite{statistical_regularity_bengio}. The authors observe that the CNN fails to generalize well when recognizability-preserving fourier domain filter masks were applied to the input. Three such data distortions were explored: no filtering, low-pass radial mask and random mask. Throughout their experiments, the authors observe that the CNN trained on the low pass filtered radially-masked inputs showed the most consistent performance across datasets, having the smallest generalization gap (among the no data-augmentation schemes). Our analysis on the spatial orderness of kernels in section \ref{sec:kernel_orderness} provides a possible explanation for this observation. Low-pass filtering enhances the spatial orderness of the input, compared to the random mask and unfiltered schemes. By doing so, it also ensures that the kernels within the architecture show greater spatial orderness; which ensures more consistent performance across data distortion variations.

\section{Conclusions}

A new statistical measure for quantifying spatial order within 2D data at various scales was proposed, called spatial orderness. This measure was shown to be indicative of the spatial organization at various scales, decreasing in value in correlation to the amount of block-swapping performed on the input. The performance gains from adding convolution layers was demonstrated to weaken with greater block-swapping disruption. Interesting bi-phasic trend in the spatial orderness of feature maps was observed during training. Theoretical and empirical results demonstrated the correlation between the spatial orderness of trained kernels, and the spatial orderness of the input. Additionally, we find that spatial orderness of kernels shows a significant drop with greater kernel-size, as it approaches a FC-NN like configuration. 
%-------------------------------------------------------------------------

\section*{Acknowledgments}
This research was supported by DSO National Laboratories, Singapore (grant no. R-719-000-029-592). We thank Dr. Loo Nin Teow and Dr. How Khee Yin for helpful discussions.

\bibliography{egbib}

\appendix

\section{Theorems} 
Please refer to section 5 in the main paper for the implications of the theorems, and empirical validation of the theoretical results.
\begin{theorem}
\label{thm1} 
Consider a layer within a Convolutional Neural Network trained by backpropagation, which contains kernel of size $(K \times K)$, s.t. $K \geq 3$. After training, let $w(p),w(q),w(r)$ be extracted tuples of kernel weight values from a particular instance of the kernel $W$, for locations $(p,q,r)$ within the kernel, following the two-hop spatial arrangement. Let the input feature map for $W$ be denoted as $X$. Consider all spatial locations within the input feature map $a$, $b$ and $c$, which follow the same two hop spatial arrangement as $p,q,r$.  
Then we must have, 
\begin{align*}
\underset{p,q}{\mathbb{E}}\left[\lvert w(p)-w(q)   \rvert ^2\right] &\leq \alpha \underset{a,b,n}{\mathbb{E}} \left [ \lvert X_n(a)-X_n(b)   \rvert ^2 \right ] , \  \ and \\ 
 \underset{p,r}{\mathbb{E}} \left[ \lvert w(p)-w(r)   \rvert ^2 \right] &\leq \alpha (1+so(X)^1) \underset{a,b,n}{\mathbb{E}} \left [ \lvert X_n(a)-X_n(b)    \rvert ^2 \right] , \,
\end{align*}
for a certain non-zero valued $\alpha$. The expectation on the left of the inequality is taken over all possible $p,q,r$ locations within the kernel, which obey the two-hop spatial arrangement. $so(X)^1$ as usual denotes spatial orderness of the input feature map at the first scale, averaged across all examples of the feature map $X_1,X_2,..X_N$.
\end{theorem}

\begin{proof}
Let $X_{rect}$  be any randomly located rectangular region of size $(K\times K)$ within the $i^{th}$ feature map which is the input to $W$, and the corresponding scalar output node denoted by $o_k$, which belongs to the $j^{th}$ feature map in the next layer. The weight update rule for the kernel $W$, only pertaining to the backpropgation error signal at $o_k$, at epoch $t$, for input $X_n$ is 
\begin{equation} 
\Delta W = -\eta(t)X_{rect}\delta_k^n.
\end{equation}

Here $\delta_k^n$ is the backprop error signal at $o_k$ and $\eta(t)$ is the gradient descent update rate. Let us denote the corresponding nodes within $X_{rect}$ by ($X_n(a)$,$X_n(b)$,$X_n(c)$) which are respectively attached to ($w(p)$, $w(q)$, $w(r)$). Note that the spatial relationship between the feature map locations $a$, $b$ and $c$ is the same as between $p$, $q$ and $r$, i.e. both are 2-hop arrangements in their respective domains. The above update rule only considers a single ouptut node for the update. If one were to consider all updates across all examples and outputs, then the final kernel value for $w(p)$ obtained after $T$ epochs can be simplified as, 
\begin{equation} \label{eq:weight_update}
    w(p) = \sum_{t,n,a}-\eta(t)\delta^n_kX_n(a) = \sum_{n,a}C_nX_n(a),
\end{equation}
where $C_n=\sum_{t,k}-\eta(t)\delta^n_k$. Similarly to equation \ref{eq:weight_update}, one can obtain forms for $w(q)$ and $w(r)$. Finally we have, 
\begin{align}
\underset{p,q}{\mathbb{E}} \left[\lvert w(p)-w(q)   \rvert ^2 \right] &= \underset{a,b,n}{\mathbb{E}}\left[( \sum_{n}C_n(X_n(a)-X_n(b)) ) ^2 \right]\\ 
&\leq   \underset{a,b,n}{\mathbb{E}}\left[\sum_{n}\left (C_n \right ) ^2  \sum_{n}\left (X_n(a)-X_n(b) \right )^2 \right] \leq \alpha \underset{a,b,n}{\mathbb{E}} \left [ \lvert X_n(a)-X_n(b)   \rvert ^2 \right ], \label{eq:ker_diffs}
\end{align}
where $\alpha =\sum_{n}\left (C_n \right ) ^2N_{total} $. Here $N_{total}$ is the total number of backpropgation driven updates on the kernel values. Similarly, for the locations $p$ and $r$ which are at a 2-hop distance, we have
\begin{equation} \label{eq:two_hop_bound}
    \underset{p,r}{\mathbb{E}}\left[ \lvert w(p)-w(r)   \rvert ^2 \right] \leq \alpha \underset{a,c,n}{\mathbb{E}} \left [ \lvert X_n(a)-X_n(c)   \rvert ^2 \right ] \leq \alpha (1+so(X)^1) \underset{a,b,n}{\mathbb{E}} \left [ \lvert X_n(a)-X_n(b)    \rvert ^2 \right],
\end{equation}
where $so(X)^1$ is the spatial orderness of the input feature map at the first scale. This completes the proof. 
\end{proof}

\begin{corollary}
\label{thm2} 
Consider a trained CNN with the same setup as in Theorem \ref{thm1}, with the same symbol definitions. Here, we vary the size of the kernel $K\times K$, with all other network parameters unchanged. For simplicity, we consider the kernels from the first layer. Thus, now each $X_i$ is simply the 2D inputs to the CNN, of size $S\times S$. We consider zero-padded convolutions in the first layer. This ensures that when $K=S$, the convolution layer simply becomes a fully connected layer. We define

\begin{equation}
D_{ab} = \underset{a,b,n}{\mathbb{E}} \left [ \lvert X_n(a)-X_n(b)   \rvert ^2 \right ] .
\end{equation} 
Let us also denote gaussian random variables
\begin{equation} 
 \epsilon_1 \sim \mathcal{N}\left(0,\frac{\sigma_1^2}{(S-K+1)^2N}\right) \  and  \ \epsilon_2 \sim \mathcal{N}\left(0,\frac{\sigma_2^2}{(S-K+1)^2N}\right),
 \end{equation}
 for certain non-zero real constants $\sigma_1$ and $\sigma_2$. 
 We wish to compute the uncertainty in the upper bounds of the kernel value differences. It follows that  

\begin{align*}
\lvert w(p)-w(q)   \rvert ^2 &\leq \alpha(K)(D_{ab}+\epsilon_1), \  \ and \\ 
 \lvert w(p)-w(r)   \rvert ^2  &\leq  \alpha(K)\left(1+so(X)^1 + \frac{\epsilon_2}{D_{ab}} \right)D_{ab}, \,
\end{align*}
for a certain non-zero valued $\alpha$(K), which is only a function of the kernel size $K$. Here $p,q,r$ are fixed kernel locations which follow the 2-hop arrangement. 
\end{corollary}

\begin{proof}
The main observation required to prove this result is that the number of updates of each kernel value is dependent on the size of the kernel, given that there are a fixed number of training examples. More precisely, for a kernel of size $K\times K$, and an input of size $S\times S$, the total number of updates, $N_{total}$, to each element of a kernel, is proportional to $(S-K+1)^2N$ (for zero-padded convolutions), s.t. one can write $N_{total}=\beta (S-K+1)^2N$. We refer the reader to equation \ref{eq:ker_diffs} from Theorem \ref{thm1}. By removing the expectation operator under $p,q$, we can reformulate the inequality as

\begin{align}
    \lvert w(p) - w(q) \rvert^2 &\leq  \left(\sum_{n}\left (C_n \right ) ^2N_{total}\right) \frac{\sum_{n=1}^{N_{total}} \lvert X_n(a) - X_n(b) \rvert^2}{N_{total}} \\ 
      &\leq \alpha(K) \left(\underset{a,b,n}{\mathbb{E}} \left [ \lvert X_n(a)-X_n(b)   \rvert ^2 \right ] + \mathcal{N}\left(0,\frac{\sigma_1'^2}{N_{total}}\right)\right) \\ 
    &\leq\alpha(K) \left( D_{ab} + \mathcal{N}\left(0,\frac{\sigma_1^2}{(S-K+1)^2N}\right)\right).
\end{align}
Here, $\sigma_1'^2$ represents the uncertainty involved in the computation of  $\underset{a,b,n}{\mathbb{E}} \left [ \lvert X_n(a)-X_n(b)   \rvert ^2 \right ]$, for $N_{total}=1$. Also, $\sigma_1^2 = \sigma_1'^2/\beta$. Similarly, we can reformulate equation \ref{eq:two_hop_bound}, to determine the upper bound on the two-hop kernel value difference as 

\begin{align}
    \lvert w(p) - w(r) \rvert^2 &\leq  \left(\sum_{n}\left (C_n \right ) ^2N_{total}\right) \frac{\sum_{n=1}^{N_{total}} \lvert X_n(a) - X_n(c) \rvert^2}{N_{total}}  \\
    &\leq\alpha(K) \left(\underset{a,b,n}{\mathbb{E}} \left [ \lvert X_n(a)-X_n(c)   \rvert ^2 \right ] + \mathcal{N}\left(0,\frac{\sigma_2^2}{(S-K+1)^2N}\right)\right) \\ 
      &\leq \alpha(K) \left(\left(1+so(X)^1\right)\underset{a,b,n}{\mathbb{E}} \left [ \lvert X_n(a)-X_n(b)   \rvert ^2 \right ] + \mathcal{N}\left(0,\frac{\sigma_2^2}{(S-K+1)^2N}\right)\right) \\ 
      &\leq  \alpha(K)\left(1+so(X)^1 + \frac{\epsilon_2}{D_{ab}} \right)D_{ab}.
\end{align}
The definitions of $\sigma_2^2$ and $\sigma_2'^2$ are analogous to $\sigma_1^2$ and $\sigma_1'^2$ before. This completes the proof. 
\end{proof}

\begin{corollary}
Consider a CNN trained with the same setup and symbol definitions as in Theorem \ref{thm1}. We bound kernel value differences averages across regions of size $d \times d$. Let us denote non-overlapping set of kernel locations $p_d$,$q_d$ and $r_d$, each of size $d\times d$, which follow the two hop spatial arrangement. Similarly we denote non-overlapping set of feature map locations in $X$ by $a_d$,$b_d$ and $c_d$, each of size $d\times d$, which follow the two hop spatial arrangement. It can then be shown that, 
\begin{align*}
\underset{p,q}{\mathbb{E}}\left[\left (  \underset{p\in p_d}{\mathbb{E}}\left[w(p)\right]-\underset{q\in q_d}{\mathbb{E}}\left[w(q)\right]   \right )  ^2\right] &\leq \alpha \underset{a,b,n}{\mathbb{E}} \left [ \left (  \underset{a\in a_d}{\mathbb{E}}\left[X_n(a)\right]-\underset{b\in b_d}{\mathbb{E}}\left[X_n(b)\right]   \right ) ^2 \right ] , \  \ and \\ 
 \underset{p,r}{\mathbb{E}} \left[ \left ( \underset{p\in p_d}{\mathbb{E}}\left[w(p)\right]-\underset{r\in r_d}{\mathbb{E}}\left[w(r)\right]   \right)  ^2 \right] &\leq \alpha (1+so(X)^d) \underset{a,b,n}{\mathbb{E}} \left [ \left (  \underset{a\in a_d}{\mathbb{E}}\left[X_n(a)\right]-\underset{b\in b_d}{\mathbb{E}}\left[X_n(b)\right]   \right )  ^2 \right]. \,
\end{align*}

Note that the above is simply a generalization of Theorem \ref{thm1} ($c=1$) to arbitrary scales. 
\label{crl2} 

\end{corollary}
\begin{proof} 
Note the final mathematical expression for $w(p)$ in equation \ref{eq:weight_update}. Extending that equation to the average of kernel values across a region of size $d\times d$, $\underset{p\in p_d}{\mathbb{E}}\left[w(q)\right]$, we have

\begin{equation}
        \underset{p\in p_d}{\mathbb{E}}\left[w(q)\right] = \sum_{t,n,a}-\eta(t)\delta^n_k \underset{a\in a_d}{\mathbb{E}}\left[X_n(a)\right] = \sum_{n,a}C_n\underset{a\in a_d}{\mathbb{E}}\left[X_n(a)\right].
\end{equation}
Subsequently, after a trivial application of Cauchy-Schwarz inequality to the expression $\underset{p,q}{\mathbb{E}}\left[\left (  \underset{p\in p_d}{\mathbb{E}}\left[w(p)\right]-\underset{q\in q_d}{\mathbb{E}}\left[w(q)\right]   \right )  ^2\right]$, similar to equation \ref{eq:ker_diffs}, the results follow. 
\end{proof}

\section{Feature Map Spatial Orderness: Progression during Training}
\label{sec:intro}
Figures 1 and 2 contains plots which depict the progression average spatial orderness of feature maps while training, for different network architecture choices. In Figure 1, the learning rate was decreased 10-fold, such that the co-occurence of phase (b) and validation accuracy increase can be more precisely observed.  

\begin{figure}
  \centering
    \includegraphics[width=0.7\textwidth]{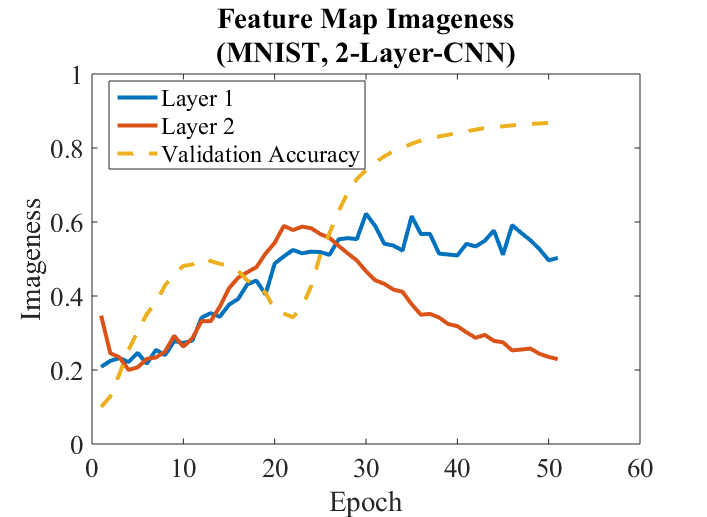}
    \caption{A CNN with two convolutional layers of $3x3$ was trained on MNIST, with a small learning rate of 0.0001. Shown are the average spatial orderness of feature maps at the end of each epoch of training, across 50 epochs. Also shown is the validation accuracy of the network at the end of each epoch. As has been observed with previous plots, the validation accuracy starts an unconstrained jump. only after the spatial orderness of the final layer (here Layer 2) starts decreasing.}  
    \label{fig:imness_seq_corr}
\end{figure}

\begin{figure}
  \centering
    \includegraphics[width=0.7\textwidth]{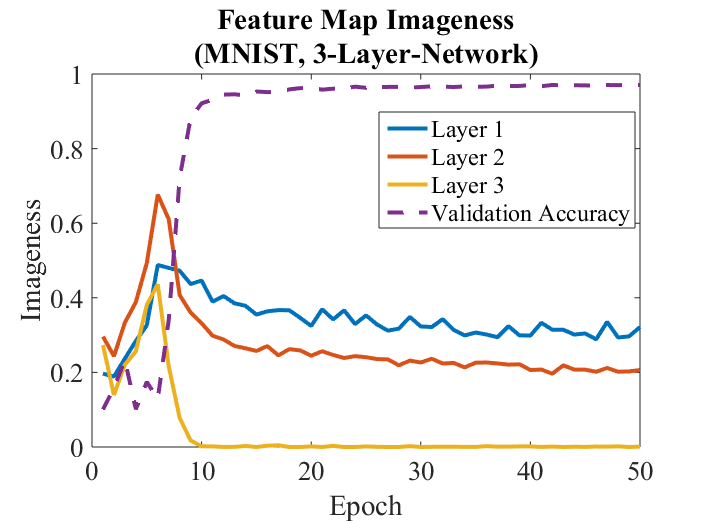}
    \caption{A CNN with three convolutional layers of $3x3$ was trained on MNIST with a learning rate of 0.001. The average spatial orderness of feature maps, along with the validation accuracy of various layers are shown across 50 epochs. }  
    \label{fig:imness_seq_corr}
\end{figure}

\section{Correlations between Spatial Orderness of Inputs and Feature Maps}

We wish to see whether any correspondence exists between the spatial orderness of the input $I_1,I_2,..I_k$ at various scales, and the spatial orderness of the feature maps at various depths. We train a CNN on the MNIST-swap(0,10,30,60) datasets, in which after every $3\times3$ convolution, a $2\times2$ max pooling operation is performed. Note that "effective" receptive field size of each position within a feature map is dependent on the depth of the feature map. Also note that max-pooling operations reduce the overlap between neighboring feature map units. These two aspects combined hint that the spatial orderness of the feature map at depth $k$ could be related to the spatial orderness of the input at a scale of $2k$. The plot in figure \ref{fig:imness_seq_corr} shows the post-training spatial orderness of feature maps (at depths 1,2,3) plotted against the spatial orderness of the input at the corresponding scale (for scales 2,4,6). We can observe that they are correlated. Also observe that reducing the spatial orderness of the input with greater block-swaps, leads to a correlated decrease of the spatial orderness of all layers. As This shows that convolutions operations are intrinsically configured to preserve the spatial order in the input.

\begin{figure}
  \centering
    \includegraphics[width=0.85\textwidth]{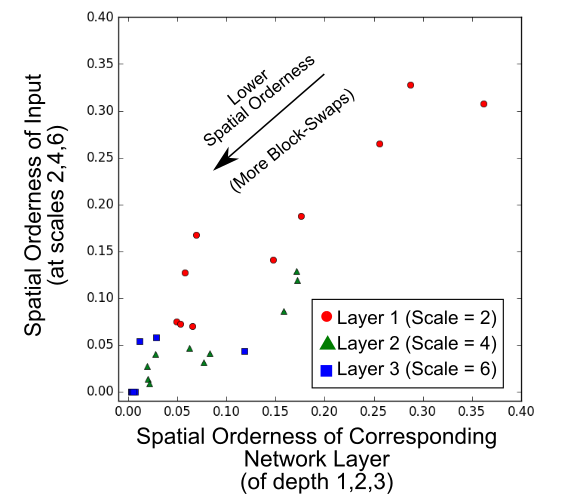}
    \caption{Correlation between spatial orderness of the input data at scales (2,4 and 6),
and the mean spatial orderness within the network feature maps at the corresponding depths (1,2 and 3), after 50
epochs of training.}  
    \label{fig:imness_seq_corr}
\end{figure}

\section{Experiments on MNIST-swap}

\begin{table}
\centering
\scalebox{0.8}{
\begin{tabular}{c|l|l|l|l|l|l|l|l|l}
\hline
\textbf{Change of}                          & \multicolumn{4}{c|}{\textbf{ Error Rate improvement} (\%)} & \multicolumn{4}{c|}{\textbf{Spatial Orderness of Final Layer}} & \textbf{Correlation} \\ \cline{2-10} 
\multicolumn{1}{l|}{\textbf{Network Depth}} & Ns=0     & Ns=10     & Ns=30    & Ns=60    & Ns=0    & Ns=10    & Ns=30   & Ns=60   &             \\ \hline
1-\textgreater{}2                  & 0.42        & 0.38         & 0.29        & 0.18        & 0.35       & 0.35        & 0.24       & 0.16       & 0.99        \\ \hline
2-\textgreater{}3                  & 0.15        & 0.14         & 0.06        & -0.09       & 0.08       & 0.09        & 0.07       & 0.04       & 0.985       \\ \hline
\end{tabular}
}
\caption{Table showing the mean reduction of validation error rate (in \%), when extra convolution layers are added to a CNN, tested and trained on MNIST. To the mid-right, the mean spatial orderness measures computed at the final convolution layer output of the CNNs are shown, before the layer addition. All block-swapping was done at Scale=3.}
\label{tab:mnist_depth_add}
\end{table}

 We wish to narrow down our hypothesis, by conjecturing that a convolution layer directly benefits from the spatial order in the feature maps that it sees. Thus, we compute the spatial orderness of the feature maps of various layers within a CNN. For this experiment, we train CNNs on datasets MNIST-swap(0,10,30,60), consisting of 1,2 and 3 convolution layers each. Next, the spatial orderness of only the final convolution layer's output of each CNN (after the max-pooling) is recorded, using the approach detailed in the previous paragraph. We also note the corresponding improvement in error rates (in \%), between CNNs containing different number of conv layers. For accurate testing, training was repeated for six trials for each CNN, and the all recorded measures (spatial orderness and error rate improvements) are averaged across all the trials. To test our narrowed hypothesis, we simply note the correlation between the spatial orderness of the final layer feature maps within a CNN, and the error improvement obtained by adding another convolution layer to the CNNs. Results are shown in Table \ref{tab:mnist_depth_add}. 

The observations are two-fold. First, we observe that reducing spatial orderness of the input leads to reduction of the spatial orderness in the feature maps, which is intuitively sound. Note that more detailed experiments on this correspondence is provided in the supplementary material. Next, we find a significant correlation between the error rate improvement from convolutional layer additions, and the spatial orderness of the final feature maps within a CNN. These observations explain the results in Section \ref{sec:swap_cnn_depth}.

\section{Variation of Spatial Orderness within feature maps}

In the experiments thus far, all feature maps within each layer were used to compute a single, mean value of spatial orderness. Here, we compute the spatial orderness measure of each feature map within the layers separately, and analyze the changes in their distribution as the network is trained. Figure \ref{fig:imness_dist} shows how spatial orderness is distributed among the units of the final layer of a 3-layer CNN being trained on MNIST, and the changes with training. To summarize the observations,

\begin{figure}
  \centering
    \includegraphics[width=0.9\textwidth]{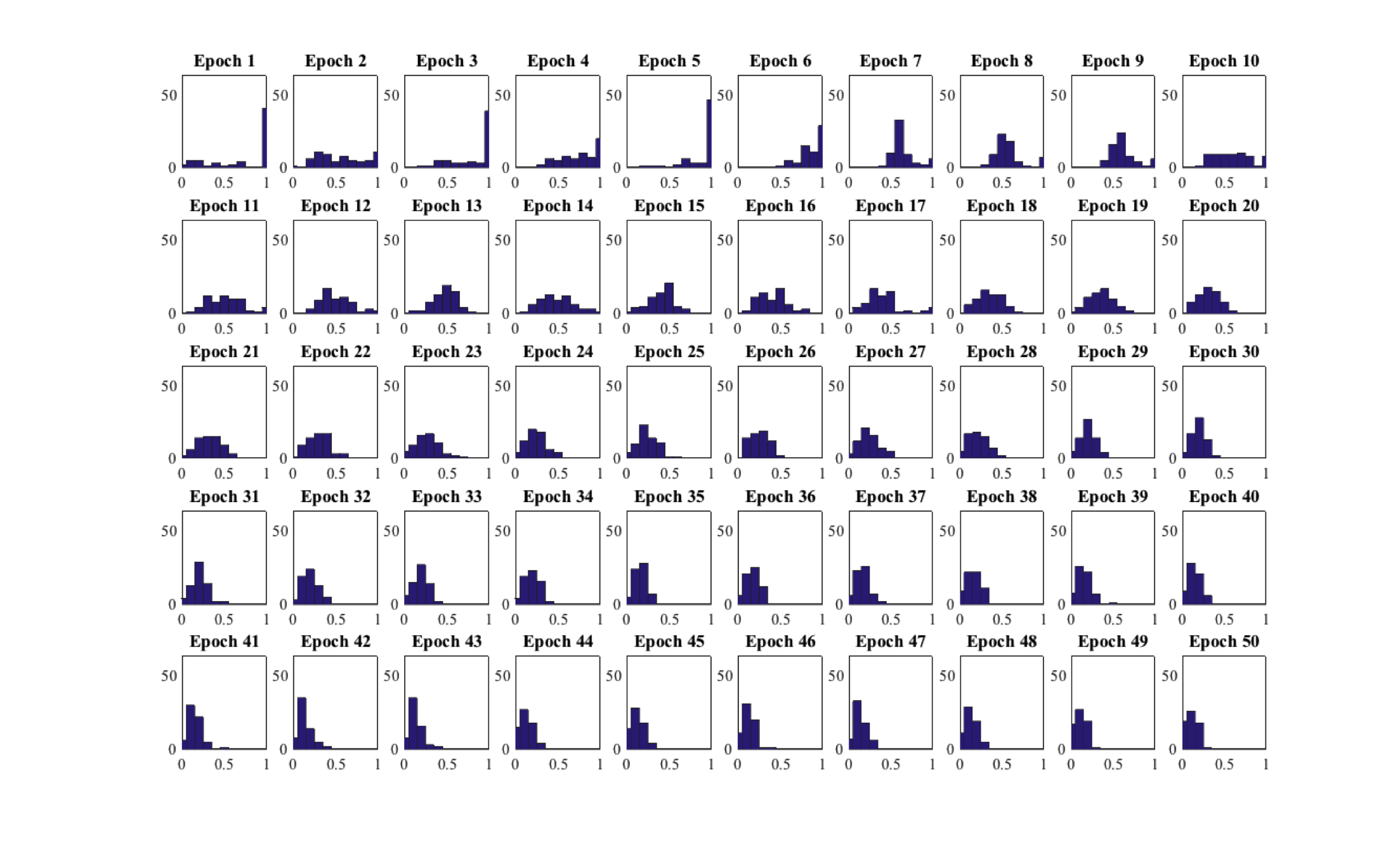}
    \caption{Shown are the histograms of the spatial orderness measures of the feature maps of the final layer of a 3-layer CNN trained on MNIST. The smooth transition from high to low spatial orderness is noticeable.}  
    \label{fig:imness_dist}
\end{figure}

% \begin{theorem}
% The diffusion process corresponding to $I_m=1$ is $P\xrightarrow{} Q \xrightarrow{} R$, whereas for $I_m=0$ it is $Q \xleftarrow{} P \xrightarrow{} R$.
% \end{theorem}

% \begin{proof}

% \end{proof}

\end{document}